\documentclass[11pt]{article}
\usepackage{fullpage}
\usepackage{mathrsfs}
\usepackage{float}
\usepackage{amsmath, amsfonts, amssymb, amsthm, amsbsy, amscd, bm, bbm} 
\usepackage{array}
\usepackage{booktabs}
\usepackage{graphicx}
\usepackage[small,bf]{caption}
\usepackage{subcaption}
\usepackage{color}
\usepackage{ifthen}
\usepackage{xspace}
\usepackage{algorithm, algorithmic}
\usepackage[colorlinks,citecolor={black},urlcolor={black},linkcolor={black}]{hyperref}
\usepackage{url}
\usepackage{tocbibind}

\DeclareGraphicsExtensions{.eps,.pdf}
\graphicspath{{figs/}}

\newcommand{\e}[1]{\mathrm{e}^{#1}}

\newcommand{\R}{\mathbb{R}}

\newcommand{\goto}{\rightarrow}
\newcommand{\iid}{i.i.d.~}

\newcommand{\T}{\top}
\newcommand\ko{\bm{X}_{\textnormal{\tiny KO}}}

\newcommand\X{\bm{X}}
\newcommand\y{\bm{y}}
\newcommand\z{\bm{z}}
\newcommand\E{\mathbb{E}}
\renewcommand\P{\mathbb{P}}
\newcommand\sgn{\operatorname{sgn}}
\newcommand{\nullx}{\textnormal{null }}
\DeclareMathOperator*{\argmin}{argmin}

\newtheorem{theorem}{Theorem}[section]

\newtheorem{proposition}[theorem]{Proposition}
\newtheorem{definition}[theorem]{Definition}
\newtheorem{lemma}[theorem]{Lemma}

\newcommand\wfdp{w\textnormal{FDP}}
\newcommand\wfdr{w\textnormal{FDR}}
\newcommand\fdp{\textnormal{FDP}}
\newcommand\tpp{\textnormal{power}}

\title{Communication-Efficient False Discovery Rate Control via \\Knockoff Aggregation}
\author{Weijie Su \and Junyang Qian \and Linxi Liu}
\date{}

\begin{document}
\maketitle

{\centering
  \vspace*{-0.3cm}  Department of Statistics, Stanford
  University, Stanford, CA 94305, USA\par\bigskip \date{November 2015}\par
}

\begin{abstract}
The false discovery rate (FDR)---the expected fraction of spurious discoveries among all the
discoveries---provides a popular statistical assessment of the reproducibility of scientific studies in various disciplines. In this work, we introduce a new method for controlling the FDR in meta-analysis of many decentralized linear models. Our method targets the scenario
where many research groups---possibly the number of which is random---are independently testing a common set of hypotheses and then sending
summary statistics to a coordinating center in an online
manner. Built on the knockoffs framework introduced by Barber and
Cand\`es (2015), our procedure starts by applying the knockoff filter
to each linear model and then aggregates the summary statistics via
one-shot communication in a novel way. This method gives exact FDR
control non-asymptotically without any knowledge of the noise
variances or making any assumption about sparsity of the signal. In certain settings, it has a communication complexity that is optimal up to a logarithmic factor.
\end{abstract}

\section{Introduction}
Modern scientific discoveries are commonly supported by statistical significance summarized from exploring datasets. In our present world of Big Data, there are a number of difficulties with this scenario: an increasing number of hypotheses tested simultaneously, extensive use of sophisticated techniques, and enormous tuning parameters. In this pipeline, spurious discoveries arise naturally by \textit{mere} random chance alone across nearly all disciplines including health care \cite{ioannidis2005,begley2012,freedman2015},  machine learning \cite{dwork2014, demvsar2006statistical}, and neuroscience \cite{kriegeskorte2009circular}.

To address this challenge, the statistical community in the past two decades has developed a variety of approaches. A landmark work \cite{BH1995} proposed the false discovery rate (FDR) as a new measure of type-I error for claiming discoveries, along with the elegant Benjamini-Hochberg procedure (BHq) for controlling the FDR in the case of independent test statistics. Roughly speaking, FDR is the expected fraction of erroneously made discoveries among all the claimed discoveries. Today, this concept has been widely accepted as a criterion for providing evidence about the reproducibility of discoveries claimed in one experiment.

Our motivation for this work is further enhanced by the observation that scientific experiments are inherently decentralized in nature, where a given set of hypotheses are probed by several groups working in parallel. For an individual group, its access to datasets collected by the others is very limited. Then, challenges arise on how to statistically and efficiently perform meta-analysis of results from all groups for controlling the FDR while maintaining a higher power (the fraction of correctly identified true discoveries) compared to an individual group. As a running example, imagine that an initiative is intended to study the genetic causes of autism across many research institutes. Due to the privacy and confidentiality of the datasets held by different institutes, it would be difficult to share full datasets. In contrast, aggregating small-volume summary statistics from each institute is a practical solution. Another issue is observed in different high-tech companies that hold background and behavioral information on thousands of millions of individuals, but are reluctant to share data for common research topics in part due to huge communication costs.

\subsection{Problem Setup and Contributions}
\label{sec:problem-setup}
To formalize the problem considered throughout the paper, suppose we observe a sequence of linear models
\begin{equation}\nonumber
  \y^i = \X^i \bm\beta^i + \z^i,
\end{equation}
where the design $\X^i \in \R^{n_i \times p}$ and the response $\y^i \in \R^{n_i}$ are collected by the $i$th group, the error term $\z^i$ has \iid $\mathcal{N}( 0, \sigma_i^2)$ entries, and the signal $\bm\beta^i \in \R^p$ may vary across different groups. Keeping in mind that a (sufficiently) strong signal in one of the groups is adequate to declare significance in the meta-analysis finding, we are interested in any feature $j$ that obeys $\beta_j^i \ne 0$ for at least one $i$; for any model selection procedure returning a set of discoveries $\widehat{S} \subset \{1, \ldots, p\}$, the false discovery proportion (FDP) is defined as 
\begin{equation}\label{eq:fdr}
\mbox{FDP} =  \frac{ \# \left\{1 \le j \le p: j \in \widehat{S} ~\mbox{and}~ \beta^i_j =0 ~\mbox{for all}~ i \right\} } {\max \{ |\widehat S |, 1 \} },
\end{equation}
and FDR is the expectation of FDP. We assume that each group has only access to its own data, that is, $(\y^i, \X^i)$, and reports summary statistics encoded in about $O(p)$ bits to a coordinating center. In this protocol, we aim to achieve the exact FDR control by only making use of the information received at the center. Our approach, referred to as \textit{knockoff aggregation}, is built on top of the knockoffs framework introduced by Barber and Cand\`es \cite{knockoff}. The knockoff filter remarkably achieves exact FDR control in the finite sample setting for a single linear model whenever the number of variables is no more than the number of observations. The validity of the method does not depend on the amplitude or sparsity of the unknown signal, or any knowledge of the noise variance. In sharp contrast, the BHq is only known to control the FDR for sequence models under very restricted correlation structures \cite{benjamini2001} apart from the independent case \cite{BH1995}. 

Some appealing features of the knockoff aggregation are listed as follows. Inherited from the knockoffs, our method also controls the FDR exactly for general design matrices in a non-asymptotic manner and does not require any knowledge of the noise variances of linear models. Apart from these inheritances, knockoff aggregation provides more refined information on the significance of each hypothesis by aggregating many independent copies of the summary statistics, resembling the multiple knockoffs as briefly mentioned in \cite{knockoff}. This property not only improves power by amplifying the signal, but also allows control of a generalized FDR which incorporates randomized decision rules. Due to the one-shot nature, this method only costs $O(p \cdot \#\mbox{linear models})$ bits in communication up to a logarithmic factor used in quantizing scalar summary statistics. We also propose a simple example where this rate of communication complexity is nearly optimal from an information-theoretic point of view.

  
\section{Preliminaries}
\label{sec:background-1}
In this section, we give a concise exposition of the knockoff filter \cite{knockoff}. Consider
\begin{equation}\nonumber
\y = \X \bm\beta + \z,
\end{equation}
where the design $\X$ is $n$ by $p$ and noise term $\z$ consists of $n$ \iid $\mathcal{N}(0, \sigma^2)$ entries. The knockoffs framework assumes the number of variables $p$ is no more than the number of measurements $n$ and the design matrix $\X$ has full rank. This is used to ensure model identifiability since otherwise there exists a non-trivial linear combination of the $p$ features $\X_j$ that sums to zero. Moreover, we normalize each column: $\|\X_j\|_2 = 1$ for all $1 \le j \le
p$.  

This method starts with constructing the knockoff features $\widetilde \X \in \R^{n \times p}$ that satisfies
\begin{equation}\label{eq:ko_symm}
\widetilde{\X}^{\T}\widetilde{\X} = \X^\T \X, \quad \X^\T \widetilde{\X} = \X^\T \X - \operatorname{diag}(\bm s),
\end{equation}
where $\bm s \in \R^p$ has nonnegative entries. To understand the
constraints, observe that the first equality forces $\widetilde \X$ to mimic the correlation structure of $\X$. The second equality further requires any original-knockoff pair $\X_j, \widetilde\X_j$ have the same correlation with all the other $2p-2$ features. In a nutshell, the purpose of knockoff design is to manually construct a control group as compared to the original design $\X$.

The next step is to generate statistics for every original-knockoff pair. Denote by $\ko = [\X, \widetilde \X] \in \R^{n \times 2p}$, the
augmented design matrix. The reference paper suggests choosing the Lasso on the augmented design as a pilot estimator:
\begin{equation}\nonumber
\widehat{\bm\beta}(\lambda) = \underset{\bm b \in \R^{2p}}{\operatorname{argmin}} ~ \frac12 \|\y - \ko \bm b\|_2^2 + \lambda\|\bm b\|_1.
\end{equation} 
Then, let $Z_j = \sup\{\lambda: \widehat\beta_j(\lambda) \ne 0\}$. Similarly, define $\widetilde Z_j$ for the knockoff variable $\widetilde\X_j$. Then, a recommended choice of the knockoff statistics are (different notation is used for the ease of exposition)
\[
W_j = \max\{Z_j, \widetilde Z_j \}, \quad \chi_j = \sgn(Z_j - \widetilde{Z}_j),
\]
where $\operatorname{sgn}(x) = -1, 0, 1$ depending on whether $x < 0, x= 0, x > 0$ respectively. As a matter of fact, many alternative knockoff statistics can be used instead, as emphasized in the reference paper. For instance, $W_j$ can take the form of any symmetric function of $Z_j$ and $\widetilde Z_j$. Furthermore, the use of the pilot estimator is not necessarily confined to the Lasso; alternatives include least-squares, least angle regression \cite{lars}, and any likelihood estimation procedures with a symmetric penalty (see e.g. \cite{SCAD,elasticnet,slope}).

The following lemma, due to \cite{knockoff}, is essential for the proof of FDR control of our knockoff aggregation. As clear from \eqref{eq:fdr}, we call $j$ a true null when $\beta_j = 0$ and a false null otherwise.

\begin{lemma}\label{lm:knockoff_key}
Conditional on all false null $\chi_j$ and all $W_j$, all true null $\chi_j$ are jointly independent and uniformly distributed on $\{-1,1\}$.
\end{lemma}

This simple lemma follows from the delicate symmetry between $\X_j$ and its knockoff $\widetilde\X_j$, which is guaranteed by the construction \eqref{eq:ko_symm}. The result implies that each $\chi_j$ can be
interpreted as a one-bit $p$-value, in the sense that it takes $1$ or $-1$ with equal probability if $\beta_j = 0$. In the case of large $|\beta_j|$, we shall expect that $\chi_j$ is more likely to take $1$ since the original feature $\X_j$ has more odds to enter the Lasso path earlier than $\widetilde\X_j$. This lemma also suggests ordering the hypotheses based on the magnitude of $W_j$ so that hypotheses that are more likely to be rejected would be tested earlier. Good ordering of hypotheses is a key element for improving power in sequential hypothesis testing (see e.g. \cite{foster2008alpha,javanmard2015}).


\section{Aggregating the Knockoffs}
\label{sec:distr-knock}
We recap the problem:  Observe a sequence of decentralized linear regression models with the same set of hypotheses of interest,
\begin{equation}\label{eq:decentralizedmodel}
\y^i = \X^i \bm\beta^i + \z^i
\end{equation}
for $1 \le i \le m$. The design matrix $\X^i$ is $n_i$ by $p$ and $\z^i$ consists of \iid centered normals. The error terms $\z^i$ are jointly independent and are allowed to have \textit{different} variance levels. The number of observed models $m$ is not necessarily deterministic, but must be independent of the randomness of all $\z^i$. We are interested in simultaneously testing
\[
H_{0,j}: \beta^1_j = \beta^2_j = \cdots = \beta^m_j = 0
\]
(that is, there is no effect of feature $j$ in all studies) versus
\[
H_{a,j}: \mbox{at least one } \beta^i_j \ne 0
\]
for $1 \le j \le p$. To fully utilize the knockoffs framework, we assume $n_i \ge p$ for all $i \ge 1$. 

Our aggregation starts by running the knockoff filter with an arbitrary pilot estimator for each linear model \eqref{eq:decentralizedmodel}, which provides us with the ordering statistics $W^i_1, \ldots W^i_p$ and the one-bit $p$-values $\chi^i_1, \ldots, \chi^i_p$. Then, for each $1 \le j \le p$, we aggregate $W^1_j, \ldots, W^m_j$ to produce $W_j$ that measures the rank of the $j$th hypothesis: the larger $W_j$ is, the earlier the hypothesis $H_{0, j}$ is tested. Let this summary statistic take the form $W_j = \Gamma(W_j^1, \ldots, W^m_j)$ for some nonnegative measurable function $\Gamma$ defined on $\R_+^m$. Recognizing that a large $W_j^i$ provides evidence of significant rank of the corresponding hypothesis, we are particularly interested in summary functions $\Gamma$ that are non-decreasing in each coordinate. No further conditions of $\Gamma$ are required. Examples include $\Gamma(x_1, \ldots, x_m) = \max\{x_1, \ldots, x_m\}, \, \Gamma(x_1, \ldots, x_m) = \mbox{the sum (or product) of the } r \mbox{ largest of } x_1, \ldots x_m$ for some $1 < r \le m$, and $\Gamma(x_1, \ldots, x_m) = \sum_{i=1}^m n_i x_i$. The first two examples are symmetric in the $W$-statistics and the last one incorporates sizes of the $m$ models.

With the ordering statistics $W_j$ in place, we move to define the aggregated $\chi$-statistics by making use of the one-bit $\chi_j^i$:
\begin{equation}\nonumber
\chi_j = \frac{m}{2} + \frac12 \sum_{i=1}^m \chi^i_j,
\end{equation}
which is simply the number of $+1$ of $\chi_j^1, \ldots, \chi^m_j$. The motivation behind this construction is simple. That is, the more winnings of the original feature $\bm X_j^i$ over its knockoffs $\widetilde{\bm X}_j^i$, the stronger evidence that $\beta_j^i$ is nonzero. As will be shown in Lemma~\ref{lm:key_lemma}, under the null $\beta_j^1 = \cdots = \beta_j^m = 0$, this aggregated $\chi_j$ follows a simple binomial distribution so that it can be easily translated into a refined $p$-value. 

In passing, the content of this section by far is summarized in Algorithm~\ref{algo:dist}.

\begin{algorithm}[H]
\caption{Running the knockoff filter in parallel}
\label{algo:dist}
\begin{algorithmic}[1]
\REQUIRE  $\X^1, \ldots, \X^m, \y^1, \ldots, \y^m$ and a summary function $\Gamma$.
\STATE Run the knockoff filter for each model and get $\chi^1_j, \ldots, \chi^m_j$ and $W_j^1, \ldots, W_j^m$.
\STATE (One-shot communication) Let $\chi_j \leftarrow  \frac{m}{2} + \frac12 \sum_{i=1}^m \chi^i_j$ and $W_j \leftarrow \Gamma(W^1_j, \ldots, W^m_j)$.
\end{algorithmic}
\end{algorithm}

\subsection{Controlling the Weighted FDR}
\label{sec:contr-fdr-other}
Having defined the aggregated knockoff statistics, we now turn to control a generalized FDR. We call it the weighted false discovery rate ($\wfdr$) which includes the original FDR as a special example. In lieu of the accept-or-reject decision rule, we introduce a randomized decision rule that assigns each hypothesis $H_{0, j}$ a number $\omega_j$ between 0 and 1. The closer $\omega_j$ is to 1, the more confidence we have in rejecting the hypothesis $H_{0, j}$. The definition the weighted FDR is given as follows.
\begin{definition}
Given a randomized decision rule $\bm\omega \in [0, 1]^p$, define the weighted false discovery proportion as
\[
\wfdp = \frac{\sum_{j=1}^p \omega_j \bm{1}_{\nullx j}}{ \sum_{j=1}^p \omega_j}
\]
if $\sum_{j=1}^p \omega_j > 0$ and otherwise $\wfdp = 0$, where $\bm{1}_{\nullx j} = 1$ if the hypothesis $H_{0, j}$ is a true null and otherwise 0. The $\wfdr$ is the expectation of $\wfdp$.
\end{definition}
If the weights $\omega_j$ take only $0, 1$, then the $\wfdr$ reduces to the vanilla FDR. In general, $\omega_j$ can be interpreted as the probability, or \textit{confidence}, of randomly rejecting $H_{0, j}$. As a special case, rejecting a hypothesis with confidence 0 is equivalent to accepting it. The motivation for this generalized FDR is simple: The accept-or-reject rule behaves like a hard rule that may make completely different decisions for very close $p$-values, whereas randomization smoothes out this undesirable artifact. From a practical point of view, this generalized FDR has the potential to find applications in large-scale Internet experiments where randomized decisions occurred frequently. Randomization of testing also comes naturally from an empirical Bayesian framework (see e.g. \cite{efron}).

As mentioned earlier, the particular form of the refined $\chi$-statistics is highly motivated by the fact that $\chi_j$ under the null hypotheses (i.e. $\beta_j^1= \cdots = \beta^m_j = 0$) are simply \iid binomial random variables, no matter how complicated the joint distribution of $W_j$ is. The following lemma formalizes this point, whose proof is just a stone away from Lemma~\ref{lm:knockoff_key}.
\begin{lemma}\label{lm:key_lemma}
Conditional on all false null $\chi_j$ and all $W_j$, all true null $\chi_j$ are jointly independent and have binomial distribution $B(m, 1/2)$.
\end{lemma}

Therefore, the refined $p$-value for testing $H_{0,j}$ is naturally given by 
\[
P_j = \frac1{2^m}\sum_{i=\chi_j}^m {m \choose i},
\]
which, by definition, is stochastically smaller than the uniform distribution on $[0, 1]$. 

Now we turn to introduce Algorithm~\ref{algo:weight} which is a generalization of the original knockoff+ filter via incorporating a confidence function $\Omega$. We call $\Omega : [0, 1] \goto [0, 1]$ a confidence function if $\Omega$ is non-increasing and obeys
\[
\lim_{x \rightarrow 0+} \Omega(x) = \Omega(0) = 1, \quad \lim_{x \rightarrow 1-} \Omega(x) = \Omega(1) = 0.
\]
This function is used to provide weights $\bm\omega$ in rejecting the hypotheses. In the special case of $\Omega(x) = \bm{1}_{x \le c}$ for some $0 < c < 1$, this algorithm reduces to the \textit{Selective SeqStep+} in \cite{knockoff}. Interested readers are referred to \cite{barber2015} for generalizations in a different direction.
Below, let $U$ be uniformly distributed on $[0, 1]$.

We present the main result as follows, which generalizes Theorem 3 of \cite{knockoff}. The control of $\wfdr$ in the finite sample setting holds for any summary function $\Gamma$ and confidence function $\Omega$.

\begin{theorem}\label{thm:fdr_control}
Combining Algorithm~\ref{algo:dist} and Algorithm~\ref{algo:weight} gives
\[
\wfdr \le q.
\]

\end{theorem}

\begin{algorithm}[H]
\caption{Knockoff filter with aggregated statistics from Algorithm~\ref{algo:dist}}
\label{algo:weight}
\begin{algorithmic}[1]
\REQUIRE  $\chi_1, \ldots, \chi_p, W_1, \ldots, W_p$ from Algorithm~\ref{algo:dist}, nominal level $q \in (0, 1)$, and confidence function $\Omega$.
\STATE Compute $P_j = \frac1{2^m}\sum_{i=\chi_j}^m {m \choose i}$.
\STATE Order hypotheses according to the magnitude of the $W$-statistics: $W_{\rho(1)} \ge W_{\rho(2)} \ge \cdots \ge W_{\rho(p)}$, where $\rho(\cdot)$ is a permutation of $1, \ldots, p$.
\STATE  Let $\widehat k$ be
\[
 \max \left\{k:  \frac{1 + \sum_{j=1}^k( 1 - \Omega(P_{\rho(j)}) )}{\sum_{j=1}^k \Omega(P_{\rho(j)})} \le  \frac{q}{\E \, \Omega(U)} - q  \right\},
\]
with the convention that $\max \emptyset = -\infty$.
\STATE Reject all hypotheses $H_{0, \rho(j)}$ for $j \le \widehat{k}$ with weight (confidence) $\omega_j = \Omega(P_{\rho(j)})$.
\end{algorithmic}
\end{algorithm}

The proof of the theorem relies on two lemmas stated below, which are parallel to Lemma 4 of \cite{knockoff}. We defer the proofs of these two lemmas to the Appendix.

\begin{lemma}\label{lm:martinglae}
Let
\[
V^+(k) = \sum_{j=1}^k \Omega(P_j) \bm{1}_{\nullx j}, \quad V^-(k) = \sum_{j=1}^k (1 - \Omega(P_j)) \bm{1}_{\nullx j}.
\]
Then,
\[
M(k) = \frac{V^+(k)}{1 + V^-(k)}
\]
is a super-martingale running backward in $k$ with respect to the filtration $\mathcal{F}_k$, which only knows all the false null $P_j$, and $V^+(k), V^+(k+1), \ldots, V^+(p)$.
\end{lemma}

\begin{lemma}\label{lm:mgle_start}
For any integer $N \ge 1$, let $U, U_1, \ldots, U_N$ be \iid uniform random variables on $[0, 1]$. Then,
\begin{equation}\label{eq:mgle_start_val}
\E\left[ \frac{\sum_{j=1}^N \Omega(U_j) }{ 1 + \sum_{j=1}^N (1 - \Omega(U_j))} \right] \le \frac{\E \, \Omega(U)}{1 - \E \, \Omega(U)}.
\end{equation}
\end{lemma}

Now, we turn to give the proof of Theorem~\ref{thm:fdr_control}. The idea of the proof is similar to that of Theorem 2 in \cite{knockoff}.
\begin{proof}[Proof of Theorem~\ref{thm:fdr_control}]
Recall that the weights $\omega_j$ are given as $\omega_j = \Omega(P_j)$ if $H_{0, j}$ is ranked among the first $\widehat k$ hypotheses processed by Algorithm~\ref{algo:weight} and otherwise $\omega_j = 0$. Due to the independence between $m$ and all $\z^i$, by a conditional argument the theorem is reduced to proving for a deterministic $m$. By Lemma~\ref{lm:key_lemma}, without loss of generality, we may assume $W_1 \ge \cdots \ge W_p$. Then we get
\begin{align*}
\wfdp \cdot \bm{1}_{\widehat k \ge 1} &= \frac{\sum_{j=1}^{\widehat k} \Omega(P_j) \bm{1}_{\nullx j}}{ \sum_{j=1}^{\widehat k} \Omega(P_j)}\\ 
&= \frac{1 + \sum_{j=1}^{\widehat k} (1 - \Omega(P_j))\bm{1}_{\nullx j}}{\sum_{j=1}^{\widehat k} \Omega(P_j)}  \cdot \frac{\sum_{j=1}^{\widehat k} \Omega(P_j) \bm{1}_{\nullx j}}{ 1 + \sum_{j=1}^{\widehat k} (1 - \Omega(P_j))\bm{1}_{\nullx j}}\\
 &\le \frac{1 + \sum_{j=1}^{\widehat k} (1 - \Omega(P_j))}{\sum_{j=1}^{\widehat k} \Omega(P_j)} \cdot \frac{\sum_{j=1}^{\widehat k} \Omega(P_j) \bm{1}_{\nullx j}}{ 1 + \sum_{j=1}^{\widehat k} (1 - \Omega(P_j))\bm{1}_{\nullx j}}\\
& \le \frac{q(1 - \E \, \Omega(U))}{\E \, \Omega(U)} \cdot  M(\widehat k).
 \end{align*}    
Recognizing that $\widehat k$ is a stopping time with respect to $\mathcal{F}$, we apply the Doob's optional stopping theorem to the super-martingale $M(k)$ in Lemma~\ref{lm:martinglae},
\begin{align*}
\wfdr &\le \frac{q(1 - \E \, \Omega(U))}{\E \, \Omega(U)} \cdot  \E \, M(\widehat k) \\
&\le \frac{q(1 - \E \, \Omega(U))}{\E \, \Omega(U)} \cdot  \E \, M(p) \\
&\le \frac{q(1 - \E \, \Omega(U))}{\E \, \Omega(U)} \cdot  \frac{\E \, \Omega(U)}{1 - \E \, \Omega(U)} = q,
\end{align*}
where the last inequality follows from Lemma~\ref{lm:mgle_start} and the observation that $\Omega(P_j)$ is stochastically dominated by $\Omega(U_j)$.

\end{proof}

\subsection{Controlling Other Error Rates}
\label{sec:contr-other-error}
While it is out of the scope of the present paper, we would like to briefly point out that the knockoff aggregation can be applied to control other type-I error rates, including the $k$-familywise error rate ($k$-FWER) \cite{hh1988}, $\gamma$-FDP \cite{genovese2004stochastic}, and per-family error rate (PFER) \cite{gordon2007}. These error rates have different interpretations from the FDR and are more favorable in certain applications. Interested readers are referred to a recent work \cite{knockoffFWER} where some attractive features of the knockoffs framework are translated into a novel procedure for provably controlling the $k$-FWER and PFER. With more refined information, the knockoff aggregation has the potential to improve power while still controlling these error rates.


\section{Communication Complexity}
\label{sec:optimality-algo}
The knockoff aggregation is communication efficient due to its one-shot nature. For each decentralized linear model, the message sent to the coordinating center is merely the sign information $\chi_j^i$ and the ordering information $W_j^i$. This piece of information can be encoded in $O(p \, \mathtt{poly}(\log p))$ bits, where the polylogarithmic factor is used in quantizing each $W_j^i$ depending on the accuracy required. Hence, the total bits of communication is $\widetilde O(mp)$. We can further get rid of this logarithmic factor by forcing $W_j^i$ to take only $0$ or $1$, respectively, depending on whether the original $W_j^i$ is below the median of the original $W_1^i, \ldots, W_p^i$ or not.

It would be interesting to get a lower bound on the total communication cost required to control the FDR while maintaining a decent power. A trivial bound as such is $\Omega(p)$ since it needs $p$ bits to fully characterize the support set of $\bm\beta^i$ (in this section $\Omega(\cdot)$ is the Big Omega notation in complexity theory, instead of the confidence function). In general, this bound is unachievable since the summary statistics from each decentralized model are obtained by using only local information. To shed light on this, we provide a simple but illuminating example of \eqref{eq:decentralizedmodel} where $\Theta(mp)$ is the optimal communication cost in achieving asymptotically vanishing FDP and full power up to a polylogarithmic factor. To start with, fix the noise level $\sigma_i^2 = 1$. All $\bm\beta^i$ are equal to a common $\bm\beta$. Let the design matrix $\bm X^i$ of each decentralized model be a $(2p) \times p$ matrix with orthonormal columns, and each $\beta_j$ independently take $\mu := \sqrt{\frac{\log p}{m}}$ with probability half and 0 otherwise. We further assume that $\bm\beta$ is independent of $\bm z^i$, and both $p, m \goto \infty$ but do not differ extremely from each other in the sense that $p = O(\e{m^{0.99}})$ and $m = O(\mathtt{poly}(p))$. This condition allows $m \asymp \log^2 p$ or $m \asymp p^{\alpha}$ for arbitrary $\alpha > 0$. Here, the summary statistics are defined as follows: We regress $\bm y^i $ on the augmented design $[\bm X^i,
\widetilde{ \bm X }^i ]$ ($\widetilde{ \bm X }^i$ is an orthogonal complement of $\bm X^i$), obtain the least-squares estimates $\widehat\beta_j^i, \widetilde\beta_j^i$ for each $1 \le j \le p$, and then take $\chi_j^i = \sgn(\widehat\beta_j^i - \widetilde\beta_j^i)$ and $W_j^i = 1$ or 0, respectively, depending on whether $|\widehat\beta_j^i - \widetilde\beta_j^i|$ is above the median of $|\widehat\beta_1^i - \widetilde\beta_1^i|, \ldots, |\widehat\beta_p^i - \widetilde\beta_p^i|$ or not.


Under the preceding assumptions, the knockoff aggregation almost perfectly recovers the support set of the signal $\bm\beta$, with a total communication cost of $O(mp)$. Let $\bm V \in \{0, 1\}^p$ be constructed as $V_j = 1$ if $\beta_j \ne 0$ and otherwise $V_j = 0$ (so $\bm V$ is uniformly distributed in the cube $\{0, 1\}^p$). Similarly, the output of the knockoff aggregation, denoted as $\widehat{\bm V}_{\textnormal{\tiny KO}}$, takes the form of $\widehat V_{\textnormal{\tiny KO}, j} = 1$ if $H_{0, j}$ is rejected and $\widehat V_{\textnormal{\tiny KO}, j} = 0$ if $H_{0, j}$ is accepted. Last, denote by $\operatorname{Hamm}(\cdot, \cdot)$ the Hamming distance.

\begin{proposition}\label{prop:knockoff_opt}
Let $\epsilon$ be any positive constant. With probability tending to one, this knockoff aggregation with slowly vanishing nominal levels $q$ obeys
\[
\operatorname{Hamm}(\widehat {\bm V}_{\textnormal{\tiny KO}}, \bm V) \le \epsilon p.
\]
\end{proposition}
Hence, the knockoff aggregation is capable of distinguishing almost all the signal features from the noise features, resulting $\fdp \goto 0$ and $\tpp \goto 1$. The nominal level $q$ is spelled out in the proof. 

Next, we move to give the information-theoretic lower bound. Denote by $\bm M^i$ the message sent by the $i$th model, which only depends on the local information $\bm y^i, \X^i$. Then the coordinating center makes decisions $\widehat{\bm V} \in \{0, 1\}^p$ to reject or accept each of the $p$ hypotheses solely based on the $m$ pieces of messages $\bm M^1, \ldots, \bm M^m$. In other words, the protocol is non-interactive. Let $L^i$ be the minimal length of $\bm M^i$ in bit, with a preassigned budget constraint $\E (L^1 + \cdots + L^m) \le B$. The proof of the result below uses tools from \cite{duchi2014}.
\begin{proposition}\label{prop:comm_lower}
Let $\epsilon$ and $C$ be arbitrary positive constants. If the total communication budget 
\[
B = \frac{Cmp}{ \log^{2.1} p},
\]
then for any non-interactive protocol,
\[
\operatorname{Hamm}(\widehat {\bm V}, \bm V) \ge \frac{1 - \epsilon}{2} \, p
\]
holds with probability tending to one.
\end{proposition}
Incidentally, the exponent $2.1$ can be replaced by any constant greater than 2. To appreciate this result, note that randomly flipping a coin for each hypothesis would have a Hamming distance about $p/2$ from the true support set $\bm V$. Hence, it is hopeless to draw any statistically valid conclusion based on $O(mp/\log^{2+o(1)} p)$ bits of information in the distributed setting. In a nutshell, for our example a communication budget of $O(mp)$ up to a logarithmic factor, is both sufficient and necessary for recovering the true signal.


\section{Numerical Experiments}
\label{simulate}

In this section, we test the performance of the knockoff aggregation under a range of designs with different sparsity levels and signal strengths.  
Recall that we have $m$ decentralized linear models
\[
\y^i = \X^i \bm\beta^i + \z^i
\]
for $i = 1, \ldots, m$, where $\X^i \in \R^{n_i \times p}$. Our setup is similar to \cite{knockoff}. First, the rows of $\X^i$ are drawn independently from $\mathcal{N} (\bm 0, \bm\Sigma)$, and are also independent of other sub-models. The columns of each $\X^i$ are then normalized to have unit length. Second, given a sparsity level $k$, we randomly sample $k$ signal locations and set $\beta^i_j = A$ for each selected index $j$ and all $i$, where $A$ is a fixed magnitude. Last, we fix the design and repeat the experiment by drawing $\y \overset{\text{iid}}{\sim} \mathcal{N} (\X \bm\beta, \bm I)$. Nominal FDR level is set to be $q = 0.20$.

\subsection{Power Gains with FDR Control}
Our first experiment tests the performance across different $m$ as the sparsity level or signal strength varies. To save space, here we only show the results for the models with independent features, i.e. $\bm\Sigma$ is diagonal. Similar patterns still hold under correlated designs. We take $p = 1000$, and $n_i = 3000$ for each $i$. Fix $\bm\Sigma = \bm I$ and $m = 5$. In this scenario, we take $A = 1.2 \sqrt{(2 {\log p})/{5}} \approx 1.99$, where $\sqrt{(2 {\log p})/{5}}$ is the universal threshold for detection if we had access to the entire datasets of the $m=5$ decentralized models, and 1.2 is a compensation factor for information loss in our communication-efficient aggregation. Each experiment is repeated 30 times.

In the knockoff aggregation, we are allowed to choose the summary function $\Gamma$  (in Algorithm~\ref{algo:dist}) and confidence function $\Omega$ (in Algorithm~\ref{algo:weight}). The choice can be made adaptively to different $m, n, p$ and the design structure. In the following simulation, we take $\Gamma (W_j^1, \ldots, W_j^m) = \sum_{i=1}^m n_i W_j^i$, and $\Omega(P_j) = \bm{1}_{P_j \leq 0.5}$.


Figure~\ref{diff_m_sparsity_iid} shows the FDR and power achieved by knockoff aggregation with fixed signal strength $A = 1.99$ and varying sparsity levels $k = 10, 30, 50, 100$, as well as with fixed sparsity $k = 30$ and varying strengths $A = c \sqrt{(2\log p)/5}$, where $c = 1.0, 1.2, 1.5, 2.0$. Recall that the power is the fraction of identified true discoveries among all the $k$ potential true discoveries. We see that our procedure can effectively control the FDR for different $m$ in both cases. Meanwhile, as $m$ increases power gains are significant even we only need $\widetilde O(mp)$ bits of communication.

\begin{figure}[!h]
\centering
\includegraphics[width=\textwidth, height = 0.3\textwidth]{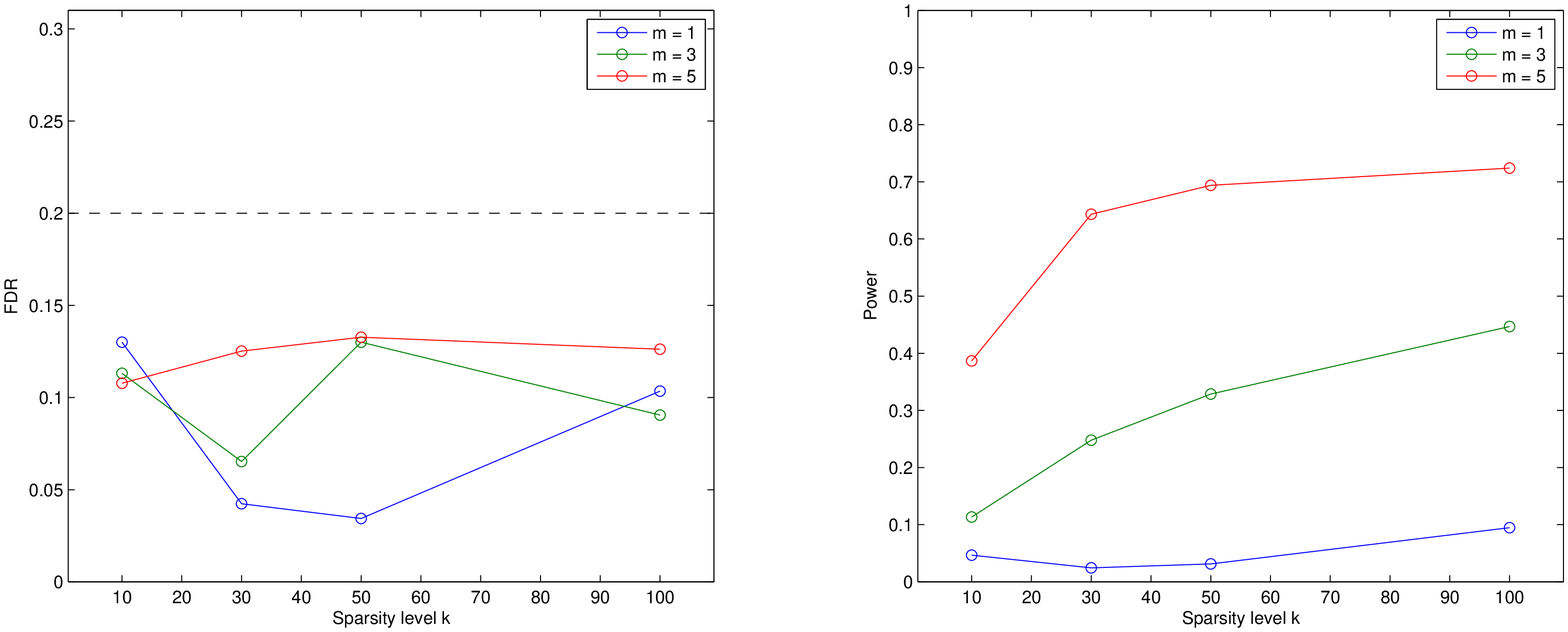}
\vspace{-0.1in}
\includegraphics[width=\textwidth, height = 0.3\textwidth]{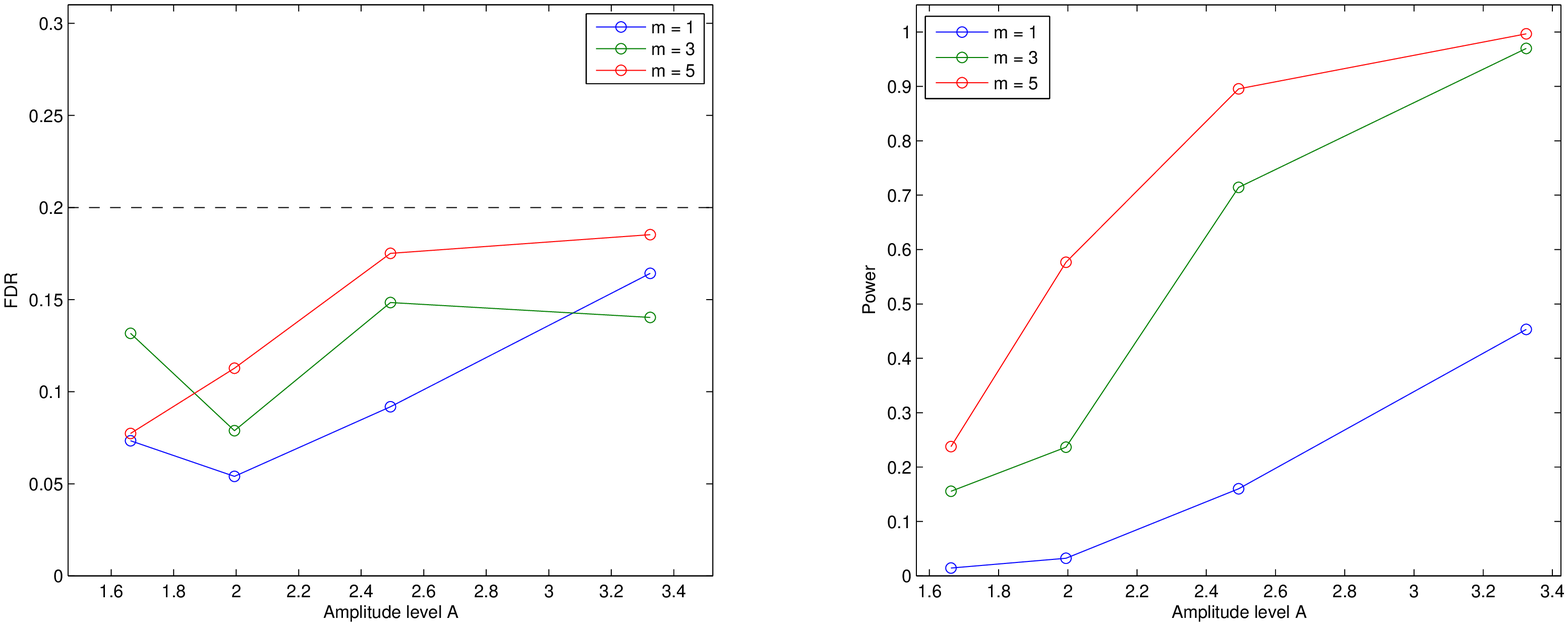}
\caption{Top row: mean FDP and power versus sparsity level $k$ with fixed strength $A = 1.99$. Bottom row: mean FDP and power versus amplitude level $A$ with fixed sparsity $k = 30$. We use i.i.d. design with $p = 1000, n_i = 3000$ and nominal level $q = 0.20$.}
\label{diff_m_sparsity_iid}
\end{figure}

\subsection{Comparison with Other Methods}
We compare the knockoff aggregation with other methods, such as the least-squares (OLS) and the Lasso. For the OLS, we consider the following procedure. For each $i$, we have the OLS estimator $\widehat{\bm \beta}^i$ based on $(\X^i, \y^i)$. This estimator obeys $\widehat{\bm \beta}^i \sim \mathcal{N} (\bm \beta, \bm\Theta^i)$, where $\bm\Theta^i = ((\X^i)^\T \X^i)^{-1}$. Then, $\widehat{\bm \beta}^i$ and the corresponding marginal variances $(\bm\Theta^i_{jj})_{1 \leq j \leq p}$ are aggregated from the $m$ nodes as follows. Let $\widehat{\beta}_j = \sum_{i=1}^m \widehat{\beta}^i_j /m$ be an averaged estimator of $\beta_j$ and $\Theta_j = (1/m^2) \sum_{i=1}^m \Theta^i_{jj}$ be its variance. Set the $z$-score $Z_j = {\widehat{\beta}_j}/\sqrt{\Theta_j}$ for testing $H_{0,j}$. Note that $Z_j \sim \mathcal{N}(0, 1)$ marginally when $\beta_j^1 = \cdots \beta^m_j = 0$. Ignoring the correlations, we apply the BHq procedure directly to the $p$-values derived from the $z$-scores. We simply call this OLS for convenience hereafter.

For the Lasso, we take the following approach. Given data $(\X^i, \y^i)$, we compute the Lasso estimates in parallel
\[
\widehat{\bm \beta}_{\text{Lasso}}^i = \argmin_{\bm b \in \R^p} ~ \frac{1}{2} \|\y^i - \X^i \bm b\|_2^2 + \lambda_i \|\bm b\|_1,
\]
where $\lambda_i \geq 0$ is a regularization parameter and is often selected by cross-validation. In particular, we choose the largest value of $\lambda_i$ such that the cross-validation error is within one standard error of the minimum. For each $i$, the support set of $\widehat{\bm \beta}_{\text{Lasso}}^i$ is sent to the center and a majority vote is applied to determine whether to accept $H_{0,j}$ or not.

To compare these methods, we consider a correlated design as an illustration. Let $p = 500, n_i = 1500$, and $m = 5$. To generate the rows $\X^i_j$, we set $\Sigma_{ij} = 1$ if $i = j$ and $\Sigma_{ij} = -0.3/(0.3 \cdot (p-2) + 1)$ otherwise. Fix the sparsity level $k = 100$ and signal strength $A = 5 \sqrt{2 \log p}$. In this setting, while the powers of these procedures are essentially 1, their behavior in FDP shown in Table \ref{fdp_table} is very distinct. 

\begin{center}
  \begin{tabular}{lcc}
\toprule
& Mean of FDP & SD of FDP \\
\midrule 
    Knockoff Aggregation & 0.1774 & 0.0493 \\ 
    OLS & 0.1433 &  0.1260 \\ 
    Lasso & 0.3458 & 0.0405 \\ 
\bottomrule
  \end{tabular}
  \captionof{table}{FDP of knockoff aggregation, OLS and Lasso.}
  \label{fdp_table}
\end{center}

The Lasso with a cross-validated penalty lacks a guarantee of FDR control (see e.g. \cite{lassoFDR}). In the case of correlated designs, its empirical FDR $0.3458$ is way higher than the nominal level $q = 0.20$. Despite the fact that we choose a sparser model in cross-validation, Lasso still tends to select more variables than necessary. In terms of controlling false discoveries, the Lasso does not give a satisfactory solution. In contrast, the mean FDP of the knockoff aggregation and that of the OLS are both under the nominal level, though the former is slightly higher than the latter. 
However, more importantly, as shown in Figure~\ref{hist_fdp}, the FDPs of the knockoff aggregation are tightly concentrated around the nominal level, while those of the OLS are widely spread---sometimes the proportions of false discoveries can be as high as 70\% for the OLS. For information, the estimated standard deviation of the knockoff FDP is 0.0493, while, in stark contrast, that of the OLS FDP is 0.1260---almost 3 times higher. Such high variability is undesirable in practice. Researchers would not like to take the risk of having 70\% false discoveries in any study.

\begin{figure}[!h]
\centering
\makebox[\textwidth][c]{\includegraphics[width=1.2\textwidth, height = 0.3\textwidth]{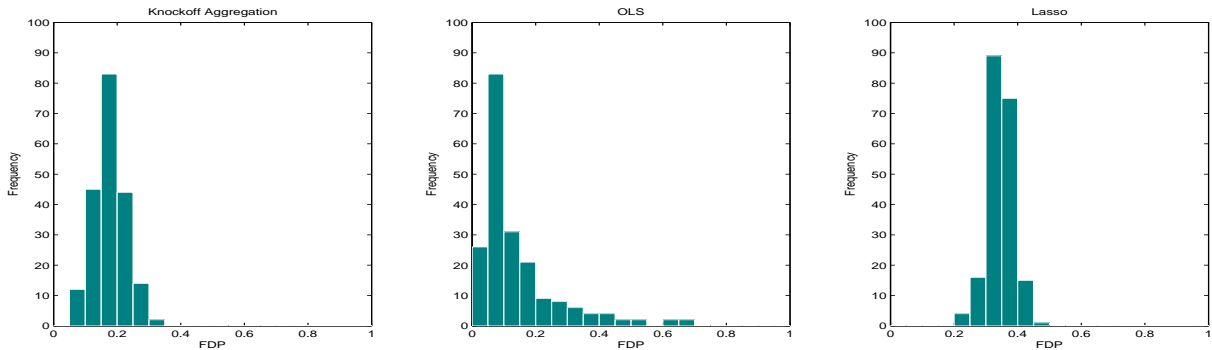}}%
\vspace{-0.1in}
\caption{Histogram of the FDPs by knockoff aggregation, OLS, and Lasso with 200 replicates. $m = 5, p = 500$, and $n_i = 1500$. Sparsity level $k = 100$ and signal strength $A = 5\sqrt{2 \log p}$.}
\label{hist_fdp}
\end{figure}


\section{Discussion}
\label{sec:discussion-1}
We introduce a communication-efficient method for aggregating the knockoff filter running on many decentralized linear models. This knockoff aggregation enjoys exact FDR control and some desired properties inherited from the knockoffs framework. Simulation results provide evidence that this proposed method exhibits nice properties in a range of examples.

Many challenging problems remain and we address a few of them. An outstanding open problem is to generalize the knockoffs framework to the high-dimensional setting $p > n$. This would as well help the knockoff aggregation cover a broader range of applications. In addition, the flexibility of the use of the link functions $\Omega$ and $\Gamma$ leaves room for further investigation: For example, does there exist an optimal $\Omega$ or $\Gamma$? Can these functions be chosen in a data-driven fashion? Last, it would be interesting to incorporate differential privacy (see e.g. \cite{cdwork2013}) in the stage of aggregation, which could lead to much stronger protection of confidentiality.

\subsection*{Acknowledgments}
W.~S.~was supported in part by a General Wang Yaowu Stanford Graduate Fellowship. W.~S.~would like to thank Emmanuel Cand\`es for encouragement. We would like to thank John Duchi for helpful comments, and Becky Richardson for discussions about an early version of the manuscript.

\bibliographystyle{abbrv}
\bibliography{ref}

\begin{thebibliography}{10}

\bibitem{knockoff}
R.~F. Barber and E.~J. Cand\`es.
\newblock Controlling the false discovery rate via knockoffs.
\newblock {\em The Annals of Statistics}, 43(5):2055--2085, 2015.

\bibitem{begley2012}
C.~G. Begley and L.~M. Ellis.
\newblock Drug development: {R}aise standards for preclinical cancer research.
\newblock {\em Nature}, 483(7391):531--533, 2012.

\bibitem{BH1995}
Y.~Benjamini and Y.~Hochberg.
\newblock Controlling the false discovery rate: A practical and powerful
  approach to multiple testing.
\newblock {\em Journal of the Royal Statistical Society. Series B
  (Methodological)}, 57(1):pp. 289--300, 1995.

\bibitem{benjamini2001}
Y.~Benjamini and D.~Yekutieli.
\newblock The control of the false discovery rate in multiple testing under
  dependency.
\newblock {\em Ann. Statist.}, 29(4):1165--1188, 08 2001.

\bibitem{slope}
M.~Bogdan, E.~v.~d. Berg, C.~Sabatti, W.~Su, and E.~J. Cand\`es.
\newblock {SLOPE} -- {A}daptive variable selection via convex optimization.
\newblock {\em The Annals of Applied Statistics}, 9(3):1103--1140, 2015.

\bibitem{cover2012}
T.~M. Cover and J.~A. Thomas.
\newblock {\em Elements of Information Theory}.
\newblock John Wiley \& Sons, 2012.

\bibitem{demvsar2006statistical}
J.~Dem{\v{s}}ar.
\newblock Statistical comparisons of classifiers over multiple data sets.
\newblock {\em Journal of Machine Learning Research}, 7:1--30, 2006.

\bibitem{duchi2014}
J.~C. Duchi, M.~I. Jordan, M.~J. Wainwright, and Y.~Zhang.
\newblock Optimality guarantees for distributed statistical estimation.
\newblock {\em arXiv preprint arXiv:1405.0782}, 2014.

\bibitem{dwork2014}
C.~Dwork, V.~Feldman, M.~Hardt, T.~Pitassi, O.~Reingold, and A.~Roth.
\newblock The reusable holdout: {P}reserving validity in adaptive data
  analysis.
\newblock {\em Science}, 349(6248):636--638, 2015.

\bibitem{cdwork2013}
C.~Dwork and A.~Roth.
\newblock The algorithmic foundations of differential privacy.
\newblock {\em Theoretical Computer Science}, 9(3-4):211--407, 2013.

\bibitem{efron}
B.~Efron.
\newblock Large-scale simultaneous hypothesis testing.
\newblock {\em Journal of the American Statistical Association}, 99(465), 2004.

\bibitem{lars}
B.~Efron, T.~Hastie, I.~Johnstone, and R.~Tibshirani.
\newblock Least angle regression.
\newblock {\em The Annals of Statistics}, 32(2):407--499, 2004.

\bibitem{SCAD}
J.~Fan and R.~Li.
\newblock Variable selection via nonconcave penalized likelihood and its oracle
  properties.
\newblock {\em Journal of the American Statistical Association},
  96(456):1348--1360, 2001.

\bibitem{foster2008alpha}
D.~P. Foster and R.~A. Stine.
\newblock Alpha-investing: {A} procedure for sequential control of expected
  false discoveries.
\newblock {\em Journal of the Royal Statistical Society: Series B (Statistical
  Methodology)}, 70(2):429--444, 2008.

\bibitem{freedman2015}
L.~P. Freedman, I.~M. Cockburn, and T.~S. Simcoe.
\newblock The economics of reproducibility in preclinical research.
\newblock {\em PLoS Biol}, 13(6):e1002165, 2015.

\bibitem{genovese2004stochastic}
C.~Genovese and L.~Wasserman.
\newblock A stochastic process approach to false discovery control.
\newblock {\em The Annals of Statistics}, pages 1035--1061, 2004.

\bibitem{gordon2007}
A.~Gordon, G.~Glazko, X.~Qiu, and A.~Yakovlev.
\newblock Control of the mean number of false discoveries, bonferroni and
  stability of multiple testing.
\newblock {\em Ann. Appl. Stat.}, 1(1):179--190, 06 2007.

\bibitem{hh1988}
G.~Hommel and T.~Hoffmann.
\newblock Controlled uncertainty.
\newblock In P.~Bauer, G.~Hommel, and E.~Sonnemann, editors, {\em Multiple
  Hypothesenprüfung / Multiple Hypotheses Testing}, volume~70 of {\em
  Medizinische Informatik und Statistik}, pages 154--161. Springer Berlin
  Heidelberg, 1988.

\bibitem{ioannidis2005}
J.~P. Ioannidis.
\newblock Why most published research findings are false.
\newblock {\em PLoS Medicine}, 2(8):e124, 2005.

\bibitem{knockoffFWER}
L.~Janson and W.~Su.
\newblock Familywise error rate control via knockoffs.
\newblock {\em arXiv preprint arXiv:1505.06549}, 2015.

\bibitem{javanmard2015}
A.~Javanmard and A.~Montanari.
\newblock On online control of false discovery rate.
\newblock {\em arXiv preprint arXiv:1502.06197}, 2015.

\bibitem{kriegeskorte2009circular}
N.~Kriegeskorte, W.~K. Simmons, P.~S.~F. Bellgowan, and C.~I. Baker.
\newblock Circular analysis in systems neuroscience: the dangers of double
  dipping.
\newblock {\em Nature Neuroscience}, 12(5):535--540, 2009.

\bibitem{barber2015}
A.~Li and R.~F. Barber.
\newblock Accumulation tests for {FDR} control in ordered hypothesis testing.
\newblock {\em arXiv preprint arXiv:1505.07352}, 2015.

\bibitem{lassoFDR}
W.~Su, M.~Bogdan, and E.~J. Cand\`es.
\newblock False discoveries occur early on the {L}asso path.
\newblock {\em arXiv preprint arXiv:1511.01957}, 2015.

\bibitem{elasticnet}
H.~Zou and T.~Hastie.
\newblock Regularization and variable selection via the elastic net.
\newblock {\em Journal of the Royal Statistical Society: Series B (Statistical
  Methodology)}, 67(2):301--320, 2005.

\end{thebibliography}

\clearpage
\appendix
\section{Technical Proofs}
\label{sec:technical-proof}
\begin{proof}[Proof of Lemma~\ref{lm:key_lemma}]
Denote by $S$ the set of indices of all false null hypotheses, that is, $S = \{1 \le j \le p: \mbox{at least one } \beta^i_j \ne 0\}$. Similarly, $S^c$ corresponds to the true null hypotheses. For each $i$, conditional on $\bm W^i := (W^1_1, \ldots, W^i_p)$ and $\bm\chi^i_{S}$, from Lemma~\ref{lm:knockoff_key} it follows that $\bm\chi^i_{S^c}$ has \iid components uniformly distributed on $\{-1, 1\}$. Since the $m$ linear models are independently generated, we thus see that  the concatenation of $\bm\chi^i_{S^c}, 1 \le i \le m$ are uniformly distributed on $\{-1, 1\}^{m|S^c|}$ conditional on all $\bm W^i$ and all $\bm\chi^i_{S}$. Then the proof immediately follows by recognizing that $W_j$ and $\chi_j$ only depend on $W_j^i, 1 \le i \le m$ and $\chi_j^i, 1 \le i \le m$, respectively. The binomial distribution follows from the fact that $\chi_j$ is simply the number of $1$ in $\chi^i_j, 1 \le i \le m$.

\end{proof}

\begin{proof}[Proof of Lemma~\ref{lm:martinglae}]
We use some ideas from the proof of Lemma 4 in \cite{knockoff}. Given the filtration $\mathcal{F}_k$, we know all $V^-(k), \ldots, V^-(p)$ since it always holds that $V^-(k') + V^+(k') = \#\{j ~ \nullx: j \le k'\}$ for all $k' \ge k$. Without loss of generality, we assume all the hypotheses are true due to the observation that $M(j)$ agrees with $M(j-1)$ if the $j$th hypothesis is true. Write $V^+(k) = \sum_{j=1}^k \Omega(P_j) = A$. By the exchangeability of $\Omega(P_1), \ldots, \Omega(P_k)$, we get
\begin{equation}\label{eq:v_mean}
\E(V^+(k-1) | \mathcal{F}_k) = \E(V^+(k-1) | V^+(k) = A) = \frac{k-1}{k}A.
\end{equation} 
Hence, we have
\begin{align*}
\E(M(k-1) | \mathcal{F}_k) &= \E(M(k-1) | V^+(k) = A)\\
&= \E\left( \frac{\sum_{j=1}^{k-1} \Omega(P_j)/k}{ 1 -  \sum_{j=1}^{k-1} \Omega(P_j)/k} \Big{|} V^+(k) = A) \right).\\
\end{align*}
To proceed, note that $x/(1- x)$ is convex for $x < 1$. Since $\sum_{j=1}^{k-1} \Omega(P_j) \in [A-1, A]$ almost surely, by the inverse Jensen inequality we get
\[
\E\left( \frac{\sum_{j=1}^{k-1} \Omega(P_j)/k}{ 1 -  \sum_{j=1}^{k-1} \Omega(P_j)/k} \Big{|} V^+(k) = A) \right) \le \frac{(A-1)/k}{1 - (A-1)/k} \eta + \frac{A/k}{1 - A/k}(1 - \eta),
\]
where $\eta = A/k$ is provided by \eqref{eq:v_mean}. Simple calculation reveals that
\[
\frac{(A-1)/k}{1 - (A-1)/k} \eta + \frac{A/k}{1 - A/k}(1 - \eta) = \frac{A}{1 + k - A} = M(k),
\]
as desired.

\end{proof}

\begin{proof}[Proof of Lemma~\ref{lm:mgle_start}]
Write $\alpha = \E \, \Omega(U) \in (0, 1)$. Note that each summand $\Omega(U_j)$ obeys $0 \le \Omega(U_j) \le 1$ and $\E \, \Omega(U_j) = \alpha$. We assert that the right-hand side (RHS) of \eqref{eq:mgle_start_val} attains the maximum if each $\Omega(U_j)$ is replaced by \iid Bernoulli random variable $B(\alpha)$. (Note that $B(\alpha)$ assumes only $0, 1$ and obeys $\E \, B(\alpha) = \alpha$.) To see this, we examine which $\Omega(U_1)$ gives the maximum conditional expectation while $\sum_{j=2}^p \Omega(U_j)$ is fixed. Write
\begin{equation}\label{eq:u1_condi}
\frac{\sum_{j=1}^N \Omega(U_j) }{ 1 + \sum_{j=1}^N (1 - \Omega(U_j))} = \frac{\Omega(U_1) + \sum_{j=2}^N \Omega(U_j) }{ N+1 - \sum_{j=2}^N \Omega(U_j) - \Omega(U_1)}
\end{equation}
For any constants $c_1 > 0, c_2 > 1$, the function $(x + c_1)/(c_2 - x)$ is convex on $[0, 1]$. Hence, the inverse Jensen's inequality gives
\[
\E \left( \frac{\Omega(U_1) + c_1}{c_2 - \Omega(U_1)} \right) \le \E \left( \frac{B(\alpha) + c_1}{c_2 - B(\alpha)} \right).
\]
Applying the last display to \eqref{eq:u1_condi}, we see that the RHS of \eqref{eq:mgle_start_val} shall never decrease if each $\Omega(U_j)$ is replaced by \iid $B(\alpha)$. As a consequence, it only remains to show
\[
\E\left( \frac{B(N, \alpha) }{ 1 + N - B(N, \alpha)} \right) \le \frac{\alpha}{1 - \alpha},
\]
which has been established in the proof of Lemma 4 in \cite{knockoff}. This finishes the proof.
\end{proof}

\begin{proof}[Proof of Lemma~\ref{prop:knockoff_opt}]
We start with the observation that
\[
\P(\chi^i_j = 1 | \beta_j = \mu) = \frac12 + (1+o(1))\sqrt{\log p/m}.
\]
Hence, $\chi_j = m/2 + (1 + o_{\P}(1)) \sqrt{m\log p}$ if $\beta_j = \mu$ for a fixed $j$. By the central limit theorem, $\chi_j = m/2 + O_{\P}(\sqrt{m})$ if $\beta_j = 0$. Set the confidence function $\Omega(x) = \bm{1}_{x \le c_p}$ for some slowly vanishing sequence $c_p$. Then, as $\log p \goto \infty$ when taking the limit $p \goto \infty$, we have 
\begin{align*}
&\frac{\#\{j: \beta_j =\mu, \Omega(P_j) = 1 \}}{p} \goto \frac12\\
&\frac{\#\{j: \beta_j = 0, \Omega(P_j) = 0 \}}{p} \goto \frac12
\end{align*}
with probability tending to one, which implies that
\[
\frac{1 + \sum_{j=1}^p( 1 - \Omega(P_{\rho(j)}) )}{\sum_{j=1}^p \Omega(P_{\rho(j)})}  \goto 1
\]
for any permutation $\rho(\cdot)$. Then the rejection rule $\widehat k$ given by Algorithm~\ref{algo:weight} takes $p$ with probability approaching one since the targeted upper bound $q/\E\Omega(U) - q = q/c_p - q > 1$ asymptotically (set $c_p \ll q = q_p$). Recognizing that for almost all $j$ the weights $\Omega(P_j) = 1$ if and only if $\beta_j \ne 0$. This is equivalent to saying that $\widehat{\bm V}$ has only a vanishing fraction of indices that do not agree with $\bm V$.

\end{proof}

\begin{proof}[Proof of Proposition~\ref{prop:comm_lower}]
For the sake of generality, replace $2.1$ by $2 + \epsilon_1$ and $\epsilon$ by $\epsilon_2$. The proof makes extensive use of Lemmas 2, 4, and 6 of \cite{duchi2014}. Take $\delta = 1/2, \sigma = 1/\mu$ in Lemma 6 of \cite{duchi2014}, and
\[
a = \sqrt{m} \log^{\frac{\epsilon_1}{4}} p.
\]
First, we show that
\begin{equation}\label{eq:small_mutual}
I(\bm V; \bm{M}) = o(p).
\end{equation}
Recognizing $n_i \equiv 1$ (this is the notation used in \cite{duchi2014}, not in our problem setup), we see that the pair $(a, \delta)$ satisfies Eqn.~(18) of \cite{duchi2014} for sufficiently large $m$. Then, combined with Lemma 4, Eqn~(19b) gives
\begin{equation}\nonumber
\begin{aligned}
&I(\bm V; \bm{M}) \le \sum_{i=1}^m I(\bm V; \bm M^i) \\
&\le 128 \frac{\delta^2 a^2}{\sigma^4} \sum_{i=1}^m H(\bm M^i) + pm h_2(p^\star) + pmp^\star\\
&\le \frac{32 \log^{2 + \frac{\epsilon_1}{2}} p}{m} \sum_{i=1}^m H(\bm M^i) + pm h_2(p^\star) + pmp^\star\\
&\le \underbrace{\frac{32 B \log^{2 + \frac{\epsilon_1}{2}} p}{m}}_{I_1} + \underbrace{pm h_2(q^\star)}_{I_2} + \underbrace{pm q^\star}_{I_3},
\end{aligned}  
\end{equation}
where the last inequality follows from Shannon’s source coding theorem \cite{cover2012}. This inequality asserts that \eqref{eq:small_mutual} is a simple consequence of
\begin{equation}\label{eq:I_three}
I_l = o(p)
\end{equation}
for all $l = 1, 2, 3$. Now we move to prove \eqref{eq:I_three}. For $l = 1$, we have
\[
I_1 = \frac{32 B \log^{2 + \frac{\epsilon_1}{2}} p}{m} = O\left( \frac{mp}{\log^{2 + \epsilon_1} p} \cdot \frac{\log^{2 + \frac{\epsilon_1}{2}} p}{m}\right) \\
= O\left(  \frac{p}{\log^{\frac{\epsilon_1}{2}} p} \right)  = o(p).
\]
For $l = 2, 3$, note that
\[
q^\star = \min\left\{ 2\e{-\frac{(a - 0.5)^2}{2\sigma^2}}, \frac12 \right\}.
\]
Since $a \gg 1$, it follows that the exponent
\[
\frac{(a - 0.5)^2}{2\sigma^2} \asymp \log^{1 + \frac{\epsilon_1}{2}} p.
\]
As a consequence, 
\[
q^\star = \exp\left( -\Theta(\log^{1 + \frac{\epsilon_1}{2}} p) \right) = o(1).
\]
Thus, $I_3 = o(I_2)$, and $I_2$ further obeys
\begin{align*}
I_2 &\asymp -pm q^\star \log q^\star\\
& \asymp \frac{pm \log^{1 + \frac{\epsilon_1}{2}} p}{\exp\left( \Theta(\log^{1 + \frac{\epsilon_1}{2}} p) \right)} = o(p),
\end{align*}
which makes use of $m = O(\mathtt{poly}(p))$. This proves \eqref{eq:small_mutual}.

Next, we proceed to finish the proof by resorting to Lemma 2 of \cite{duchi2014}. To this end, set $t = (0.5 - \epsilon_2) p$ (if $\epsilon_2 \ge 0.5$ then there is nothing to prove). Then, Lemma 2 yields
\begin{equation}\label{eq:continu_fano}
\P\left( \operatorname{Hamm}(\widehat {\bm V}, \bm V) >  (0.5 - \epsilon_2) p \right) \\
\ge 1 - \frac{I(\bm V; \bm M) + \log 2}{ \log \frac{2^p}{N_t}},
\end{equation}
where, in our setting, $N_t = \{\bm v \in \{0, M\}^p: \|\bm v\|_0 \le (0.5 - \epsilon_2) p\}$. By the large deviation theory, we get
\begin{equation}\label{eq:ldt}
\log\frac{2^p}{N_t} \sim ( \log 2 + (0.5 - \epsilon_2)\log(0.5 - \epsilon_2) 
+ (0.5+\epsilon_2)\log(0.5+\epsilon_2) ) p \asymp p.
\end{equation}
Substituting \eqref{eq:small_mutual} and \eqref{eq:ldt} into \eqref{eq:continu_fano}, we obtain
\[
\P\left( \operatorname{Hamm}(\widehat {\bm V}, \bm V) >  (0.5 - \epsilon_2) p \right) \ge 1 - o(1),
\]
which completes the proof.

\end{proof}


\end{document}